\documentclass[runningheads,a4paper]{llncs}

\usepackage{amssymb}
\usepackage{graphicx}

\usepackage{url}

\usepackage{graphicx}
\usepackage{amsmath} 
\usepackage{mathtools}
\usepackage{color}
\usepackage{bm}
\usepackage{array}
\usepackage{hyperref}
\usepackage{float} 
\usepackage[percent]{overpic}

\newcommand{\figtablewidth}{0.23cm}

\renewcommand{\div}[1]{\operatorname{div} \left ( #1 \right )}
\newcommand{\tr}[1]{\mathrm{tr} \left ( #1 \right )}
\renewcommand{\det}[1]{\mathrm{det} \left ( #1 \right )}

\newcommand{\eg}[0]{\emph{e.g. }}
\newcommand{\ie}[0]{\emph{i.e. }}
\newcommand{\dive}{\,\mathrm{div}}
\newcommand{\figwidth}{0.23}

\begin{document}

\mainmatter  

\title{A Tensor Variational Formulation of \\ Gradient Energy Total Variation}

\titlerunning{Gradient Energy Total Variation}

\authorrunning{F. {\AA}str{\"o}m et al.} 

\author{Freddie {\AA}str{\"o}m\inst{1,2}, George Baravdish\inst{3} and Michael Felsberg\inst{1}} 

\institute{Computer Vision Laboratory, Link{\"o}ping University, Sweden \and  
 Center for Medical Image Science and Visualization (CMIV), Link{\"o}ping University \and
 Department of Science and Technology, Link{\"o}ping University, Sweden
}

\toctitle{}
\tocauthor{}
\maketitle

\begin{abstract}
	We present a novel variational approach to a tensor-based total variation formulation which is called \emph{gradient energy total variation}, GETV. We introduce the gradient energy tensor~\cite{diva2:269100} into the GETV and show that the corresponding Euler-Lagrange (E-L) equation is a tensor-based partial differential equation of total variation type. Furthermore, we give a proof which shows that GETV is a convex functional. This approach, in contrast to the commonly used \emph{structure tensor}, enables a formal derivation of the corresponding E-L equation. Experimental results suggest that GETV compares favourably to other state of the art variational denoising methods such as \emph{extended anisotropic diffusion} (EAD)~\cite{diva2:543914} and \emph{total variation} (TV)~\cite{Rudin1992259} for gray-scale and colour images.
\end{abstract}

\section{Introduction}
The variational approach to image diffusion is to model an energy functional $E(u) = F(u) + \lambda R(u)$ where $F(u)$ is a fidelity term. The positive constant $\lambda$ determines the influence of $R(u)$, the regularization term describing smoothness constraints on the solution $u^*$ that minimizes $E(u)$. In this work we are interested in tensor-based formulations of the regularization term and we introduce the functional \emph{gradient energy total variation} (GETV). 

A basic approach to remove additive image noise is to convolve the image data with a low-pass filter \eg a Gaussian kernel. The approach has the advantage that noise is eliminated, but so is image structure. To tackle this drawback, Perona and Malik~\cite{Perona90scale-spaceand} introduced an edge stopping function to limit the filtering where the image gradient takes on large values. A tensor-based extension of the Perona and Malik formulation was presented by Weickert~\cite{Weickert96anisotropicdiffusion} in the mid 90s which we refer to, in accordance to Weickert, as anisotropic diffusion.

The principle of anisotropic diffusion is that smoothing of the image is performed parallel to image structure. The concept is based on the structure tensor~\cite{diva2:274026,Forstner87}, a windowed second moment matrix, which describes the local orientation in terms of a tensor field. To smooth the image parallel to the image structure, the tensor field is transformed by using a non-linear diffusivity function. The transformed tensor field is then used in the diffusion scheme and the resulting tensor is commonly denoted as the diffusion tensor.
\begin{figure}[tb]
    \centering
	\includegraphics[width=1\textwidth]{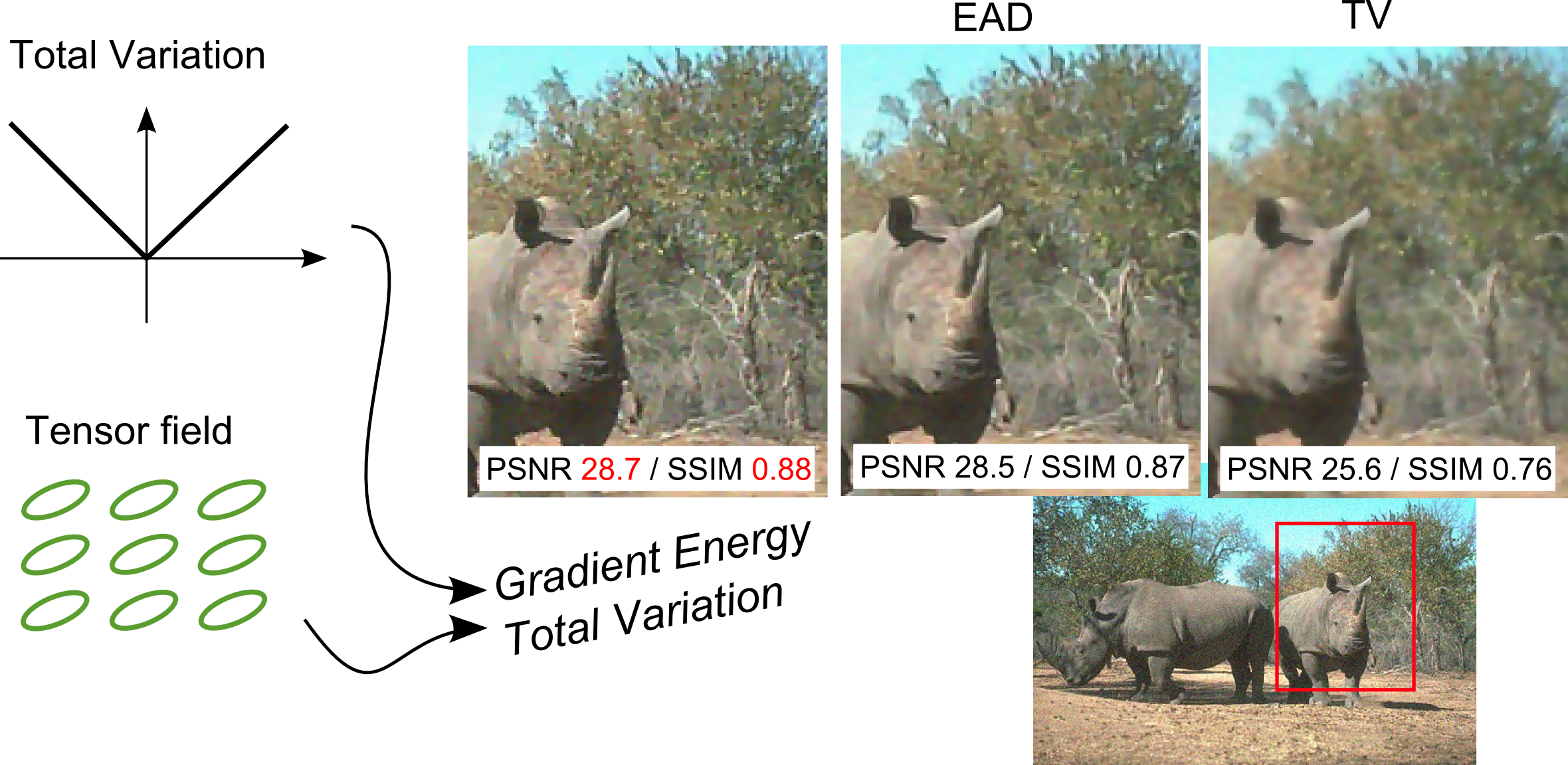}
    \caption{In this work we derive a gradient energy tensor-total variation (GETV) scheme. In this example, our approach clearly boosts the visual impression, PSNR and SSIM performance over EAD and TV in a colour image \emph{auto-denoising} application.}
    \label{fig:overview}
\end{figure}
In this paper we propose to replace the structure tensor and introduce the \emph{gradient energy tensor} (GET)~\cite{diva2:269100} into the regularization term of a total variation energy functional. The formulation allows us to consider \emph{both} the eigenvalues and eigenvectors, in contrast to previous work which only considers the eigenvalues of the structure tensor. We formulate a gradient energy total variation functional and show significant improvements over current variational state-of-the-art denoising methods. In figure~\ref{fig:overview} we illustrate the denoising result obtained by the proposed gradient energy total variation formula compared to extended anisotropic diffusion (EAD)~\cite{diva2:543914} and total variation (TV)~\cite{Rudin1992259}, note that our approach obtain higher error measures and sharper edges.

\subsection{Related works}
The linear diffusion scheme (convolution with a Gaussian kernel) is the solution of a partial differential equation (PDE) and it is closely related to the notion of scale-space~\cite{Iijim62}. Therefore it has been of interest to investigate also the Perona and Malik formulation and its successors of adaptive image filtering in terms of a variational framework.

In the linear diffusion scheme the regularization terms are given by $R(u) = \int_\Omega |\nabla u|^2 \; dx $ and in total variation (TV) $R(u) = \int_\Omega |\nabla u| \; dx$, see~\cite{Rudin1992259}. The total variation formulation is of particular interest since it has a tendency to enforce piecewise smooth surfaces, however it is also the drawback since it produces cartoon-like images.

Several works have investigated generalizations of the standard total variation approach to tensor-based formulations. Roussous and Maragos~\cite{Roussos2010}, Lefkimmiatis et al.~\cite{Lefkimmiatis2013} and Grasmair and Lenzen~\cite{Grasmair210}, all consider the \emph{structure tensor} and define objective functions in terms of the tensor eigenvalues. The difference from those work compared to our presentation is that our formulation allows us to consider both eigenvalues and eigenvectors of the \emph{gradient energy tensor}. 

Roussous and Maragos~\cite{Roussos2010} considered a functional which indirectly takes the eigenvalues of the structure tensor into account and ignores the eigenvectors. They considered the regularization $R(u) = \int_\Omega \psi(\mu_1,\mu_2) \; dx$, where $\mu_{1,2}$ are the eigenvalues of the structure tensor. They remark that standard variational calculus tools are not applicable to derive the Euler-Lagrange equation. The problem arise when computing the structure tensor where a smooth kernel is convolved with the image gradients. Furthermore, Lefkimmiatis et al. which, similar to Roussous and Maragos, considered Schatten-norm of the structure tensor eigenvalues. Grasmair and Lenzen~\cite{Grasmair210} defined $R(u) = \int_\Omega \sqrt{\nabla^tu A(u) \nabla u} \; dx$, where $A(u)$ is the structure tensor with remapped eigenvalues. The objective function is then solved using a finite element method instead of deriving a variational solution. Krajsek and Scharr~\cite{Krajsek2010}, linearized the diffusion tensor thus they obtained a linear anisotropic regularization term resulting in an approximate formulation of a tensor-valued functional for image diffusion. 

The common formulation by the aforementioned works is that they use the structure tensor which does not allow for an explicit formal derivation of the Euler-Lagrange equation. Our framework does.

\subsection{Contributions}
The approach we take in this work is to introduce the \emph{gradient energy tensor} (GET)~\cite{diva2:269100} into the regularization term of the proposed functional. We give a proof which shows that the GETV is a convex functional. Our formulation allows us to differentiate \emph{both} the eigenvalues and eigenvectors since the GET does \emph{not} (in this work) contain a post-convolution of its components. The following major contributions are presented
\begin{itemize}
	\item We present a novel objective function \emph{gradient energy total variation} which models a tensor-based total variation diffusion scheme by using the gradient energy tensor in section~\ref{sec:introducing_GET}.
	\item In section~\ref{sec:experiments}, we show that the new scheme combines EAD~\cite{diva2:543914} and TV~\cite{Rudin1992259} achieving highly competitive results for grey and colour image denoising. 
\end{itemize}

\section{Variational approach to image enhancement}
\label{sec:functionals}

\subsection{Energy minimization}
\label{sec:DIFFUSION}
The variational framework of image diffusion is based on functionals of the form
\begin{equation} \label{eq:E_diffusion}
	E(u) =  \int_\Omega (u - u^0)^2 \; d {x} + \lambda R(u),
\end{equation}
where ${x} = (x_1, x_2)^t \in \Omega \subset \mathbb{R}^2$, $\Omega$ is the image size in pixels and $u^{0}$ is the observed noisy image. The first term in~\eqref{eq:E_diffusion} is the fidelity term $F(u)$ and the second term is the regularization term, $R(u)$. The constant $\lambda > 0$ determines the amount of regularization. The stationary point that minimizes $E(u)$ is given by the Euler-Lagrange (E-L) equation
\begin{equation}
	\lim_{\varepsilon \rightarrow 0} \frac{\partial E (u + \varepsilon v)}{\partial \varepsilon}  = 0, \quad \mbox{in} \quad  \Omega,
  \label{eq:EL_min}
\end{equation}
where $v$ is a test-function. The corresponding boundary condition to~\eqref{eq:EL_min} is a homogeneous Neumann condition \eg $\nabla u \cdot {n} = 0$ where ${n}$ is the normal vector on the boundary $\partial \Omega$ and $\nabla = ( \partial_{x_1}, \partial_{x_2} )^t$ is the gradient operator. The E-L equation for total variation~\cite{Rudin1992259} is
\begin{equation}
  \left \{
    \begin{array}{rlll}
      u-u^{0} - \lambda \div{\dfrac{\nabla u}{|\nabla u|}} & = 0 & \mbox{in} & \Omega \\
      \nabla u \cdot {n} & = 0 & \mbox{on} &  \partial \Omega , \\
    \end{array} \right .
  \label{eq:EL_tv}
\end{equation}
and the $u$ which minimizes \eqref{eq:EL_tv} can be obtained by solving a parabolic initial value problem (IVP) to get the diffusion equation. Alternatively, TV is often solved by using primal-dual formulations~\cite{Chambolle_PD_TV2001} or by modifying its norm to include a constant offset to avoid discontinuous solutions.

\subsection{Tensor-based anisotropic diffusion}
\label{sec:tensor_diffusion}
In order to filter parallel to the image structure, a tensor-based anisotropic diffusion scheme was introduced by Weickert~\cite{Weickert96anisotropicdiffusion}. The filtering scheme is defined as the partial differential equation (PDE) 
\begin{equation} \label{eq:EL_AD}
	\div{D(T) \nabla u} = 0,\quad \mbox{in} \quad  \Omega.
\end{equation}
The adaptivity of the filter is determined by the structure tensor 
\begin{equation} \label{eq:structure_tensor}
	T = w*(\nabla u \nabla^t u),
\end{equation}
where $*$ is a convolution operator and $w$ is a Gaussian kernel~\cite{Forstner87,diva2:274026}. The tensor is a windowed second moment matrix, thus it estimates the local variance and can be thought of as describing a covariance matrix~\cite{Lindeberg96_direct_computation}. The eigenvector of $T$ corresponding to the largest eigenvalue is aligned orthogonal to the image structure. Therefore, to avoid blurring of image structures, the diffusion tensor $D$ is computed as $D(T) = U^T g(\Lambda) U \enspace$, where $U$ is the eigenvectors and $\Lambda$ the eigenvalues of $T$~\cite{Fberg2011}. We require $g(r)\rightarrow 1$ as $r\rightarrow 0$ and $g(r) \rightarrow 0$ as $r \rightarrow \infty$ and a common choice is the negative exponential function $g(r) = \exp(-r/k)$ where $k$ is an unknown edge-stopping parameter.

\section{Gradient energy tensor}
The gradient energy tensor is a real-valued and symmetric tensor and it determines the directional energy distribution of the signal gradient~\cite{diva2:269100}. In contrast to the structure tensor~\eqref{eq:structure_tensor} it does \emph{not} require a post-convolution of the tensor-components to form a rank 2 tensor. Note that due to the convolution operator, the structure tensor is not sensitive to structures smaller than the width of the averaging filter used to compute it.  The classical GET is defined in terms of the image data in $u$~\cite{diva2:269100}. Here we use an alternative, but equivalent, formulation of GET expressed in the image gradient. Let $H = \nabla \nabla^t$ be the Hessian and $\nabla\Delta u = \nabla \nabla^t \nabla u$, then we define the GET as
\begin{equation}\label{eq:GET}
	GET(\nabla u) = H uH u - \frac{1}{2}(\nabla u [\nabla \Delta u]^t + [\nabla \Delta u]\nabla^t u).
\end{equation} 
The presence of second and third-order derivatives in GET makes it sensitive to noise, however, it allows us to capture orientation of structures that are not possible to detect with the structure tensor. In general the GET is not positive semi-definite. An investigation of the positivity of the 1-dimensional energy operator was done in~\cite{275632}. In the two-dimensional case, the positivity of the operator is reflected in the sign of the eigenvalues. Let the components of the GET be $a$, $b$ and $c$ \ie $GET \coloneqq \begin{pmatrix} a & b \\ b & c \end{pmatrix}$, 
then the GET is positive semi-definite if the condition in Lemma~\ref{lemma:1} is satisfied.
\begin{lemma}\label{lemma:1}
	The GET is positive semi-definite if $\tr{HuHu} - \nabla^t u \nabla \Delta u \geq \sqrt{l}$
	where $l = \tr{GET}^2 - 4\det{GET} \geq 0 $.
\end{lemma}
\begin{proof}
	 Since GET is symmetric it has real eigenvalues. Thus by its eigenvalue decomposition it is sufficient to show that $\tr{GET} \geq \sqrt{l}$ in order for GET to be positive semi-definite. $l$ is necessarily positive since $l = (a-c)^2 + 4b^2 \geq 0$.
\end{proof}
Since $GET$ is not necessarily positive semi-definite we define $GET^+$, in the below definition, which is a positive semi-definite tensor.
\begin{definition}
	The positive semi-definite tensor $GET^+$ is 
	\begin{equation}\label{eq:get+}
		GET^+(\nabla u) = V^t \begin{pmatrix} |\iota_1| & 0 \\ 0 & |\iota_2| \end{pmatrix} V
	\end{equation}
	where $V$ is the eigenvectors and $\iota_{1,2}$ are eigenvalues of $GET$.
\end{definition}

\section{Introducing Gradient Energy Total Variation}
\label{sec:introducing_GET}
In this section we introduce the proposed energy functional which results in the gradient energy tensor-based total variation scheme. The regularization term we consider is given in the following definition
\begin{definition}\label{def:GETV}
	The gradient energy total variation functional (GETV) is 
	\begin{equation} \label{eq:R_GETV}
		R(u) = \int_\Omega \nabla^t u S(\nabla u) \nabla u \; d{x},
	\end{equation} 
	where $S(\nabla u) \in \mathbb{R}^{2\times 2}$ is a symmetric positive semi-definite tensor. 
\end{definition}

\subsection{Variational formulation of gradient energy total variation}
In this section we will study properties and interpretation of the GETV before deriving its corresponding Euler-Lagrange equation in the next section. 

We begin our analysis by putting $S(\nabla u) \in \mathbb{R}^{2\times 2}$ to be the symmetric positive semi-definite tensor in~\eqref{eq:R_GETV} with eigenvalues $\mu_{1,2}$ and orthonormal eigenvectors $v, w$ then  
\begin{equation*}
	S(\nabla u) = vv^t\mu_1 + ww^t\mu_2.
\end{equation*} 
Furthermore, we define a tensor $W(\nabla u) \in \mathbb{R}^{2\times 2}$ which also is symmetric positive semi-definite with corresponding eigenvectors to $S(\nabla u)$ \ie
\begin{equation*}
	W(\nabla u) = vv^t\lambda_1 + ww^t\lambda_2,
\end{equation*}
and $\lambda_{1,2}$ are the eigenvalues. In particular we will consider $W(\nabla u)$ of the form
\begin{equation}\label{eq:S}
	W(\nabla u) = |\nabla u|S(\nabla u),
\end{equation}
such that~\eqref{eq:R_GETV} is convex (see Corollary~\ref{cor:phi}). 

Start by expressing the quadratic form defined in $S(\nabla u)$ by its eigendecomposition and rearranging the resulting vectors such that
\begin{align}
\notag	 \nabla^t u S(\nabla u) \nabla u 
	& = \nabla^t u [ \mu_1 vv^t + \mu_2 ww^t ] \nabla u  \\
\label{eq:phi_orient}	& = \mu_1 v^t [\nabla u \nabla^t u] v + \mu_2 w^t [\nabla u \nabla^t u] w ,
\end{align} 
The product $\nabla u\nabla^t u$ is a rank-1 tensor with orthonormal eigenvectors $p=(p_1,p_2)^t$ and $p^\perp = (p_2, -p_1)^t$ such that $P = (p, p^\perp)$ and $\Lambda$ has the corresponding eigenvalues $\kappa_1$ and $\kappa_2$ on its diagonal. Note that, by the spectral theorem, the eigendecomposition of $\nabla u\nabla^t u$ is always well-defined, \ie the eigenvector $p$ is not singular. This is shown by the generalized definition of $p$, \ie in the case of $|\nabla u| \neq 0$ then $p=\nabla u/|\nabla u|$, and in the case of $|\nabla u| = 0$ we let $P = I$ where $I$ is the identity matrix

In the following we substitute the eigendecomposition of $\nabla u \nabla^t u = P^t\Lambda P$ into~\eqref{eq:phi_orient}:
\begin{align}
\notag	\nabla^t u S(\nabla u) \nabla u  
\notag	&=  \left ( \mu_1 v^tP \Lambda P^t v + \mu_2 w^tP \Lambda P^t w \right ) \\
\label{eq:quad_nabla_eig}	&=  \left ( \mu_1 v^tP \begin{pmatrix} \kappa_1 & 0 \\ 0 & \kappa_2 \end{pmatrix} P^t v + \mu_2 w^tP \begin{pmatrix} \kappa_1 & 0 \\ 0 & \kappa_2 \end{pmatrix} P^t w \right ) 
\end{align}
and insert $\kappa_1 = |\nabla u|^2$ and $\kappa_2 = 0$ and use the eigenvalue relation from~\eqref{eq:S} \ie $\lambda_1 = \mu_1|\nabla u|$ and $\lambda_2 = \mu_2|\nabla u|$. Then after rewriting~\eqref{eq:quad_nabla_eig} we obtain 
\begin{equation}\label{eq:phi_mu}
	 \nabla^t u S(\nabla u) \nabla u   
	= \lambda_1 |\nabla u| (v\cdot p)^2 + \lambda_2|\nabla u|(w\cdot p)^2. 
\end{equation}
The interpretation of~\eqref{eq:phi_mu} is that $v,p,w$ are normalized eigenvectors such that the scalar products defines the rotation of $W(\nabla u)$ and $S(\nabla u)$ in relation to the image gradients direction. This can be illustrated by using the definition of the scalar product \ie $v \cdot p = \cos(\theta)$ and $w \cdot p = \sin(\theta)$, where $\theta$ is the rotation angle as shown in figure~\ref{fig:eigenvectors} a. Note that, if $W(\nabla u)$ describes the local directional information its eigenvectors will be parallel to the orthonormal eigenvectors of $\nabla u\nabla^t u$, \ie $v\,||\,p$ and $w\,||\,p^\perp$ if $\theta = 0$. 

\begin{figure}[t]
	\centering
	\begin{overpic}[width=0.40\textwidth]{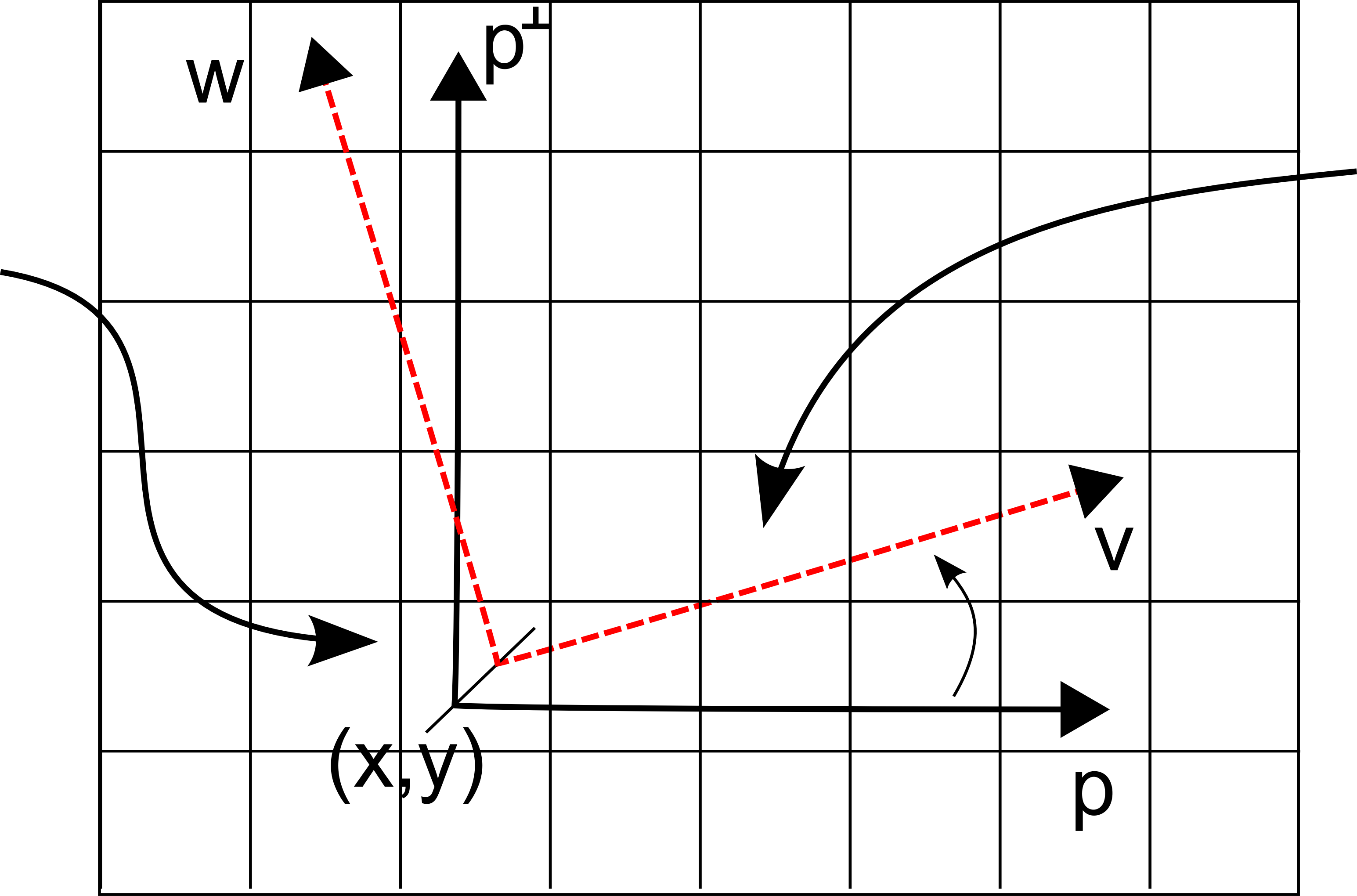}
		\put (-10,50) {$\nabla u \nabla^t u$}
		\put (75,19) {$\theta$}
		\put (65,58) {$W(\nabla u), S(\nabla u)$}
	\end{overpic}
	\hspace{10mm}
	\begin{overpic}[width=0.36\textwidth]{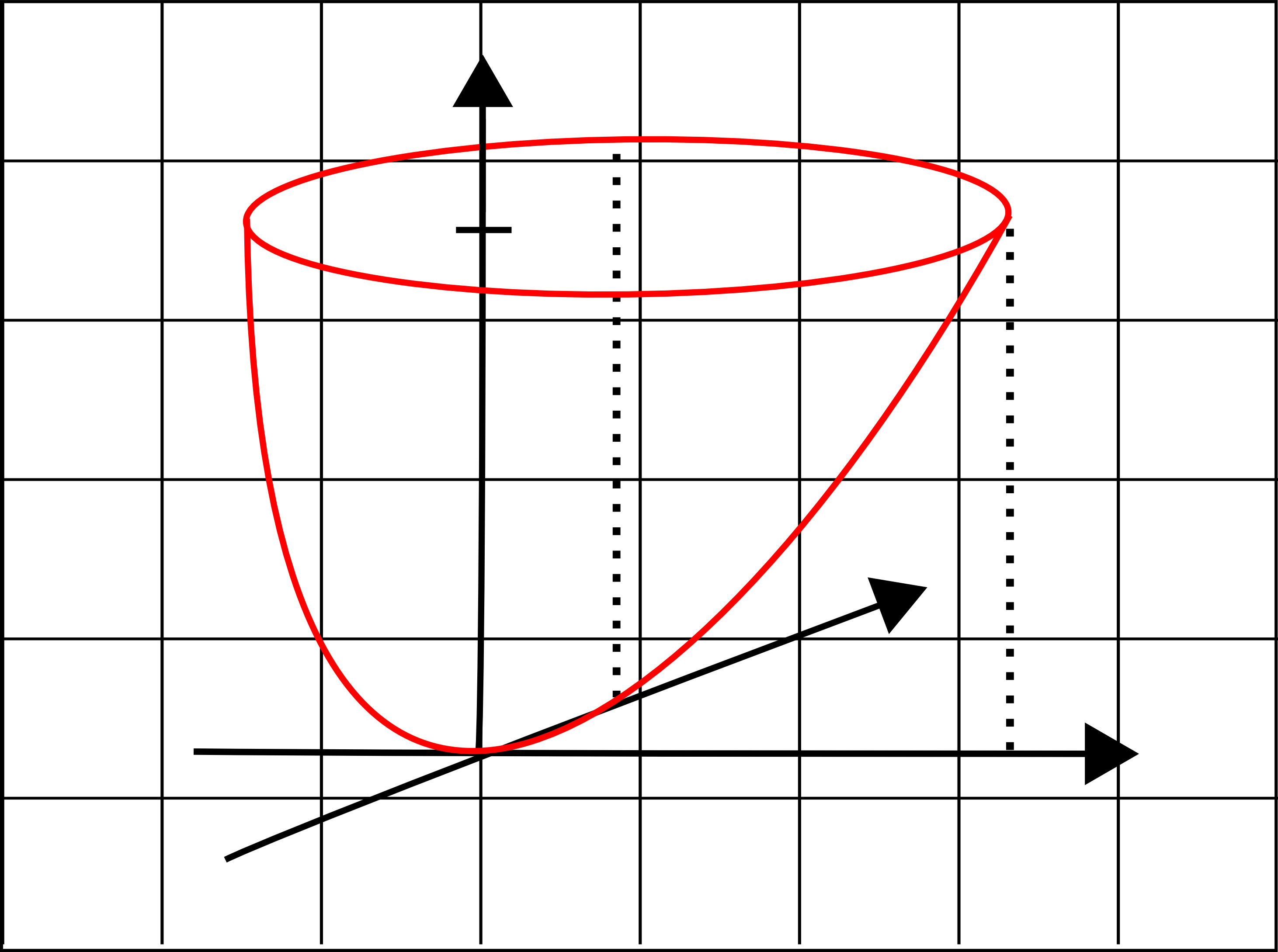}
		\put (42,67) {$\Phi(\nabla u)$}
		\put (42,55) {$c$}
		\put (35,25) {$\sqrt{\frac{\tau_2}{c}}$}
		\put (71,25) {$q_2$}
		\put (68,6) {$\sqrt{\frac{\tau_1}{c}}$}
		\put (88,10) {$q_1$}
	\end{overpic} \\
	\hfill  (a) \hfill \hfill (b) \hfill \phantom{I} \\
	\caption{(a) Illustration of eigenvector basis at coordinate $(x,y)$, the dashed (red) arrows indicate the eigenvectors of $W(\nabla u)$ and $S(\nabla u)$ and thick (black) arrows $\nabla u\nabla^t u$. (b) Illustration of the paraboloid~\eqref{eq:paraboloid} where we set $c = \tau_1q_1^2 + \tau_2q_2^2$. }
	\label{fig:eigenvectors}
\end{figure}
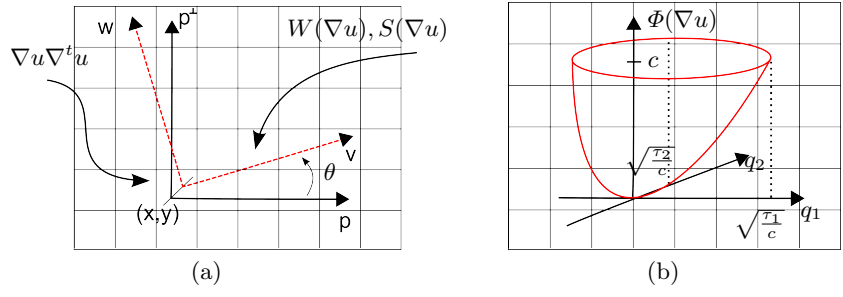

In the following we set $W(\nabla u)$ in~\eqref{eq:S} as
\begin{equation}\label{eq:W}
		W(\nabla u) = \exp(-GET^+(\nabla u)/k),
\end{equation}
where $\exp$ is the matrix exponential function such that $\lambda_i = \exp(-|\iota_i|/k)$ for $i = 1,2$ and $k>0$, the eigenvalues $\iota_i$ were defined in~\eqref{eq:get+}. In the below Corollary we put $W(\nabla u)$ as~\eqref{eq:W} and show that $\Phi(\nabla u)$ is convex and thereby $R(u)$ is convex in $u$.
\begin{corollary}\label{thm:S}
	 The GETV functional, $R(u)$, is convex w.r.t. $u$. 
\end{corollary}
\begin{proof}
To prove the convexity of $R(u)$ we write $\Phi(u) = \nabla^t u S(\nabla u) \nabla u$ in terms of the eigenvectors and eigenvalues of $W(\nabla u)$. Then, from \eqref{eq:phi_mu} it follows that 
\begin{align}
\notag	\Phi(\nabla u) & = |\nabla u|(\lambda_1 p^t (vv^t) p + \lambda_2 p^t (ww^t) p ) \\
\notag				   &=  |\nabla u|p^t(\lambda_1 vv^t + \lambda_2 ww^t)p \\
					   &=   (Vp)^t \begin{pmatrix} \tau_1 & 0 \\ 0 & \tau_2 \end{pmatrix} V p ,
\end{align}
where $V = (v, w)$ and $\tau_i(\nabla u) = |\nabla u|\lambda_i \geq 0$ for $i = 1,2$. Let $q = Vp = (q_1, q_2)^t$, then
\begin{equation}\label{eq:paraboloid}
	\Phi(\nabla u) = \tau_1 q_1^2 + \tau_2 q_2^2,
\end{equation}
is a quadratic form in the basis of orthonormal eigenvectors $V$ and $\tau_i(\nabla u)$. This quadratic form is always well-defined due to the spectral theorem. Since the paraboloid~\eqref{eq:paraboloid} has positive curvature everywhere, and $u$ maps continuously to the paraboloid, $R(u)$ is convex in $u$ which concludes the proof. \qed
\end{proof}
We illustrate $\Phi(\nabla u)$ by the paraboloid in figure~\ref{fig:eigenvectors} b.

\begin{remark}\label{cor:phi}
	 $\Phi(\nabla u)$ can also be expressed in the Schatten-1 norm \cite{Horn:1985:MA:5509} (pp. 441) \ie $||A||_1 = \sum_i^n |\sigma_i(A)| $ where $\sigma_i(A)$ denotes the $i$'th singular value of a tensor $A$. Since $||A||_1$ is a norm it has the important properties of positivity and convexity. However, in our case it is not obvious that the convexity follow directly from the norm due to the non-linearity of $W(\nabla u)$, see Corollary~\ref{thm:S}. From \eqref{eq:phi_mu} we have that
	\begin{equation}\label{eq:S_schatten}
		\nabla^t u S(\nabla u) \nabla u = ||A(\nabla u)||_1.  
	\end{equation}
	This means that $A(\nabla u)$ has singular values $\sigma_1(A) = \lambda_1 |\nabla u| (v\cdot p)^2 $ and $\sigma_2(A) = \lambda_2|\nabla u|(w\cdot p)^2$ where $\lambda_{1,2}$ are the eigenvalues of $W(\nabla u)$ in~\eqref{eq:S}. 
\end{remark}

\begin{remark}
	The standard total variation formulation is obtained from~\eqref{eq:S} by setting $W(\nabla u)$ as the identity matrix $I$, then we have  
	$
		\Phi(\nabla u)  
		= \nabla^t u \frac{I}{|\nabla u|} \nabla u 
		= \frac{|\nabla u|^2}{|\nabla u|} 
		= |\nabla u|
	$.	
	Notice that we can derive the same result from~\eqref{eq:S_schatten}. Suppose that $W(\nabla u) = I$, then $\lambda_{1,2}=1$ and $\sigma_1(A) = |\nabla u|$ and $\sigma_2(A) = 0$, thus we obtain that 
	$
		 \Phi(\nabla u) = ||A(\nabla u)||_1 = |\nabla u|.
	$
\end{remark}

\subsection{A formal minimizer of gradient energy total variation}
In the previous section we defined the GETV in definition~\ref{def:GETV} by putting $S(\nabla u)$ according to~\eqref{eq:S} with $W(\nabla u)$ as in~\eqref{eq:W}. In order to minimize our proposed functional we use a result from~\cite{diva2:543914} which derived the corresponding Euler-Lagrange equation for a functional with a quadratic form. Therefore we use this result to directly minimize~\eqref{eq:Rfunctional} in the below Theorem~\ref{thm:R} in order to compute the Euler-Lagrange equation of~\eqref{eq:R_GETV}. Note that Theorem~\ref{thm:R} is restated from~\cite{diva2:543914} but with the difference that the tensor $S$ is symmetric.   

\begin{theorem}\label{thm:R}
Let the regularization term $R(u)$ in the functional~\eqref{eq:E_diffusion}, be given by
\begin{equation}
	R(u) = \int_{\Omega} \nabla^t u S(\nabla u) \nabla u \; d{x},
    \label{eq:Rfunctional}
\end{equation}
where $u \in \mathcal{C}^2$ and $S(\nabla u) \in \mathbb{R}^{2\times 2}$ is a tensor-valued function $\mathbb{R}^2 \rightarrow \mathbb{R}^{2\times2}$. Set ${\bm s} = \nabla u$ and define
\begin{equation} \label{eq:T_s}
S_{\bm s}({\bm s}) =
 \begin{pmatrix} 
    \nabla^t u S_{s_1}({\bm s})  \\
    \nabla^t u S_{s_2}({\bm s}) 
  \end{pmatrix} .
\end{equation}
where $S_{s_1}$ is defined as the component-wise differentiation of $S$ with respect to $s_1$. Then the corresponding E-L equation is given by
\begin{equation}
   \left \{
    \begin{array}{rl}
     \div{ B \nabla u}  &  = 0 \mbox{ in }  \Omega \\[3mm]
      {n} \cdot B\nabla u & =  0     \mbox{ on }  \partial \Omega, 
     \end{array} \right . 
   \label{eq:ELexp_simplified_2}
 \end{equation}
 where $B = [2S + S_{\bm s} ]\big |_{\bm{s}=\nabla u}$.
\end{theorem}

By using Theorem~\ref{thm:R} we compute $S_{\bm{s}}$ as
\begin{align} \label{eq:corr1_Ts}
	S_{\bm s} = 
	-\frac{1}{|\nabla u|^3}
	\begin{pmatrix} 
		u_x \nabla^t u  W \\ 
		u_y \nabla^t u  W
	 \end{pmatrix}	
	+	 
	\frac{1}{|\nabla u|}
	\begin{pmatrix} 
		\nabla^t u  W_{u_x} \\ 
		\nabla^t u  W_{u_y}
	 \end{pmatrix} ,
\end{align}
so that the corresponding minimizer of the regularizer~\eqref{eq:Rfunctional} is obtained by inserting ~\eqref{eq:corr1_Ts} into~\eqref{eq:ELexp_simplified_2} 
\begin{equation}
   \left \{
    \begin{array}{rl}
	   \div{ 
	   \underbrace{
		\left [
			2W  
			+	 
			\begin{pmatrix} 
				\nabla^t u  W_{u_x} \\ 
				\nabla^t u  W_{u_y}
			 \end{pmatrix}
		   -\dfrac{1}{|\nabla u|^2}
			\begin{pmatrix} 
				u_x \nabla^t u  W \\ 
				u_y \nabla^t u  W
			 \end{pmatrix}	
	   \right ] 
	   }_{=Q}
	  	\dfrac{\nabla u}{|\nabla u|}}  &  = 0 \mbox{ in }  \Omega \\[3mm]
      {n} \cdot B\nabla u & =  0     \mbox{ on }  \partial \Omega, 
     \end{array} \right . 
   \label{eq:ELexp_simplified_final}
 \end{equation}
where the bracket, which we denote as $Q$, defines a weight controlling the anisotropy of the total variation scheme. We compute the component-wise derivative of $W$ with respect to $u_x$ and $u_y$ by using an explicit eigendecomposition \ie
\begin{align} \label{eq:D_ux}
\notag	W_{u_x}
	= &
	\left [ 
		\begin{pmatrix} 
			2 v_1   &   v_2    \\
			v_2     &  0
		\end{pmatrix}(\partial_{u_x} v_1) 
		+
		\begin{pmatrix} 
			0   &   v_1    \\
			v_1 &  2v_2
		\end{pmatrix}(\partial_{u_x} v_2) 
	\right ] 
	\lambda_1 + vv^t(\partial_{u_x} \lambda_1) \\
	& +  
	\left [ 
		\begin{pmatrix} 
			2 w_1   &   w_2    \\
			w_2     &  0
		\end{pmatrix}(\partial_{u_x} w_1) 
		+
		\begin{pmatrix} 
			0   &   w_1    \\
			w_1 &  2w_2
		\end{pmatrix}(\partial_{u_x} w_2) 
	\right ] 
	\lambda_2 + ww^t(\partial_{u_x} \lambda_2),
\end{align}
with the corresponding orthonormal eigenvectors $v$ and $w$. The general expressions for the derivatives of the eigenvalues and eigenvectors are given in the supplementary material. 

The most intuitive interpretation of the GETV is to consider the  eigendecomposition of $W(\nabla u)$. Thus given eigenvalues $\lambda_{1,2} = \exp(-|\iota_{1,2}|/k)$ of $W(\nabla u)$ where $\iota_{1,2}$ are eigenvalues computed from the gradient energy tensor, the exponential function will adapt the filtering to be parallel to the image structures \ie close to an image structure $\lambda_1$ will be small and $\lambda_2$ larger. Since the gradient energy tensor does not contain a post-convolution of the tensor-components, our formulation allows us to better preserve fine details in the image structure, than if we would use the structure tensor, as we show in the numerical experiments. 

\subsection{Discretization}

The proposed PDE~\eqref{eq:ELexp_simplified_final} is solved with a forward Euler-scheme and the image derivatives are approximated by using regularized finite differences~\cite{JWeickert_Disc}. The numerical approximation of the divergence operator is based on the expansion
$ \dive (M \nabla u) = \partial_x( M_{11}\partial_x u ) + \partial_x ( M_{12}\partial_y u)
   + \partial_y(M_{21}\partial_x u) + \partial_y(M_{22} \partial_y u).
$
The first and last term in the previous equation are computed by averaging the forward $\partial^+$ and backward $\partial^-$ finite difference operators, and the mixed derivatives are computed with central differences. The final E-L equation that we solve is~\eqref{eq:ELexp_simplified_final} but with regularized derivatives, \ie let $\beta$ denote regularization with a small positive constant such that the denominators are expressed as $|\nabla u|_\beta = \sqrt{|\nabla u|^2 + \beta^2}$. Furthermore, we are required to compute third-order derivatives (in GET) for terms such as $\partial_x \Delta u$, and we found that it is appropriate to directly approximate these higher-order derivatives with central differences of the Laplacian. In practice, to avoid numerical instabilities, it is sufficient to regularize the first and second order derivatives with a Gaussian filter of standard deviation $\sigma$ of $8/10$. To compute the third-order term a Gaussian filter of standard deviation $3$ was suitable for regularization. These filter sizes were kept constant for all images and all noise levels in the experimental evaluation, we fixed $\beta = 10^{-4}$. 

\begin{figure}[tb]
    \centering
	\begin{minipage}{0.49\textwidth}
		\centering
		\includegraphics[width=0.49\textwidth]{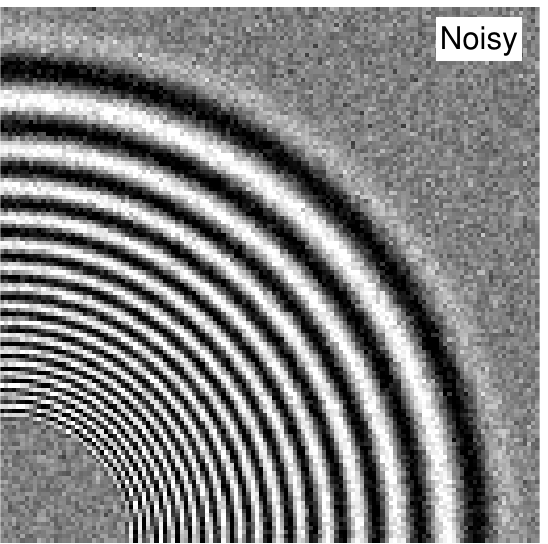}
		\includegraphics[width=0.49\textwidth]{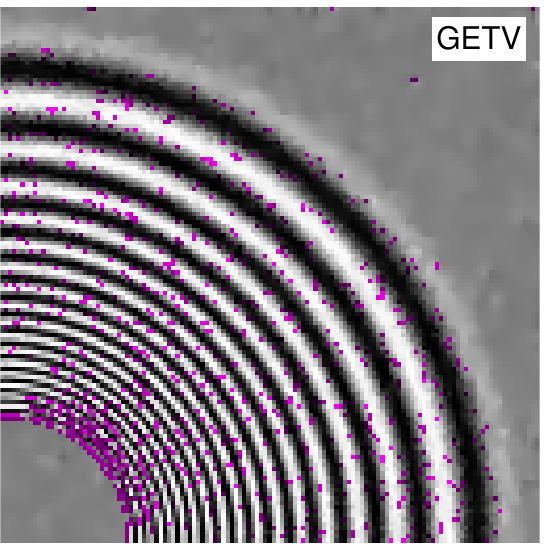} \\[1mm]
		\includegraphics[width=0.49\textwidth]{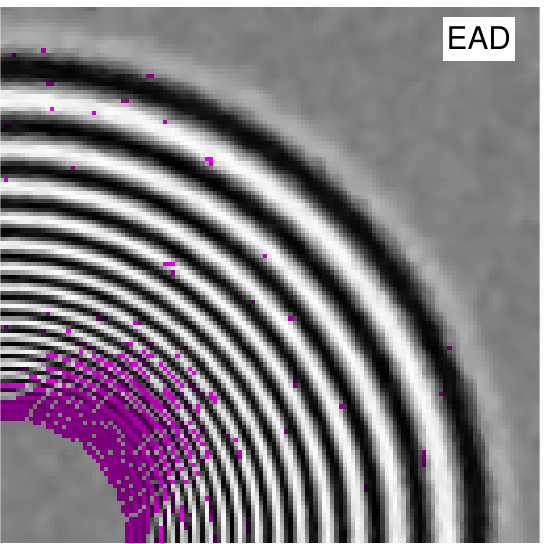} 
		\includegraphics[width=0.49\textwidth]{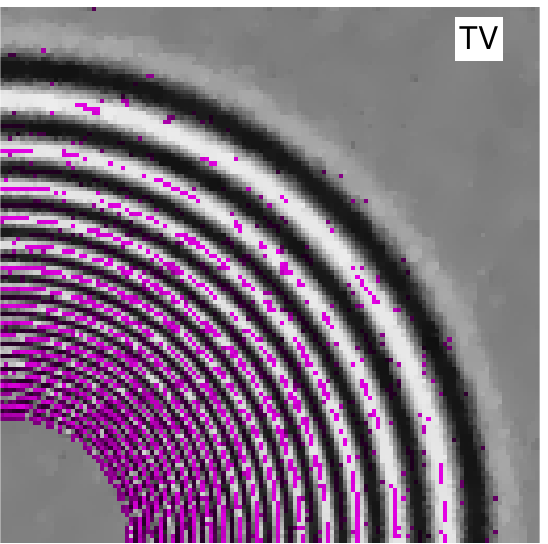} 
	\end{minipage}
	\begin{minipage}{0.49\textwidth}
		\includegraphics[width=0.95\textwidth]{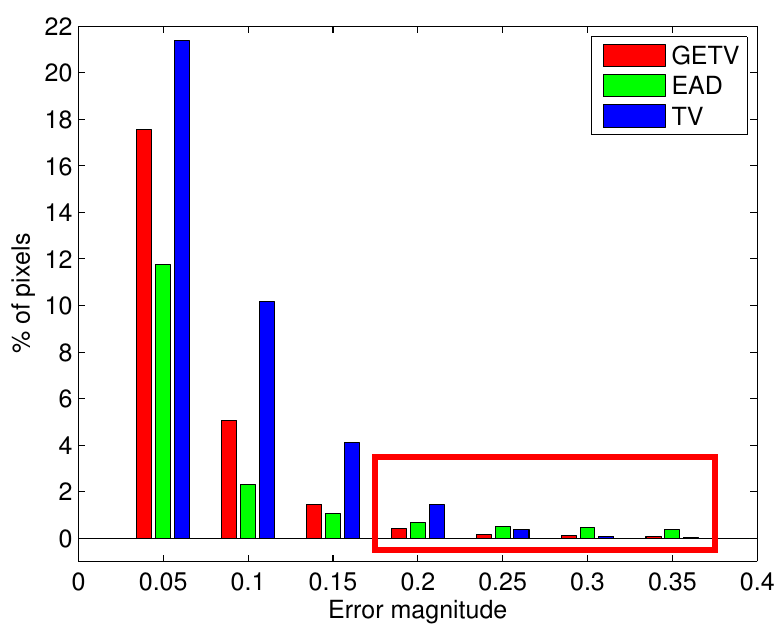} 
	\end{minipage}
    \caption{A test image corrupted with $20$ standard deviation additive Gaussian noise and the corresponding denoised results. Observe that the errors larger than $10\%$ (magenta) are considerably less concentrated at high frequencies for the GETV than the other methods.}
    \label{fig:illustration_TV_AD}
\end{figure}

\section{Application to image enhancement}
\label{sec:experiments}

We evaluate our approach with respect to extended anisotropic diffusion (EAD)~\cite{diva2:543914} and a state-of-the-art primal-dual implementation of the Rudin-Osher-Fatemi~\cite{Rudin1992259} total variation model (TV)~\cite{Chambolle_PD_TV2001}\footnote{Code (14-02-17) \url{gpu4vision.icg.tugraz.at/index.php?content=downloads.php}}. In figure~\ref{fig:illustration_TV_AD} we illustrate the behaviour of the three schemes on a radial test pattern consisting of increasingly high-frequency components. The histogram illustrates that the proposed gradient energy total variation (GETV) scheme in essence exhibits fewer large magnitude errors than the other methods, this is marked by the red box. The EAD scheme shows errors in high-frequency areas as illustrated with the magenta colour, whereas standard total variation gives errors for all frequencies due to the tendency of enforcing piece-wise constant surfaces.  

\renewcommand{\figwidth}{0.24}
\begin{figure}[t]
    \centering
	\includegraphics[width=0.99\textwidth]{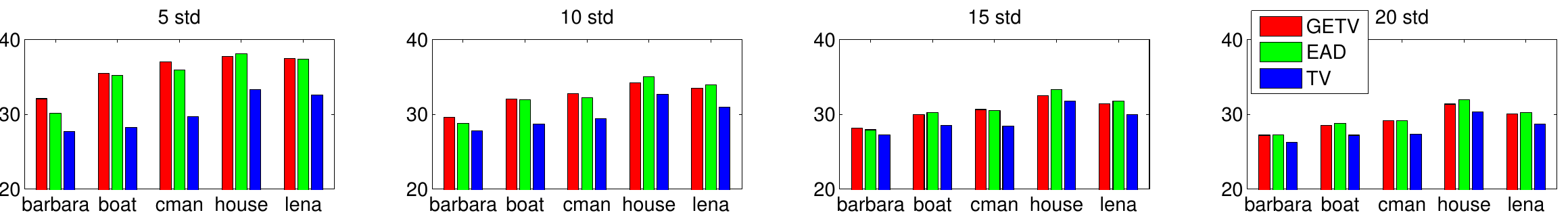} \\
	\textbf{(a)}  PSNR values for the grey-scale images.   \\[4mm]
	\includegraphics[width=\figwidth\textwidth]{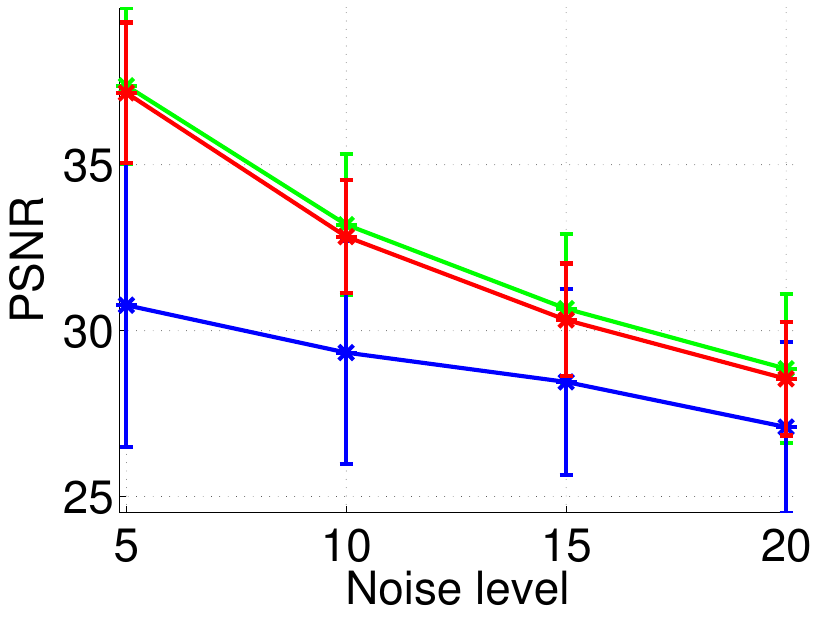}
	\includegraphics[width=\figwidth\textwidth]{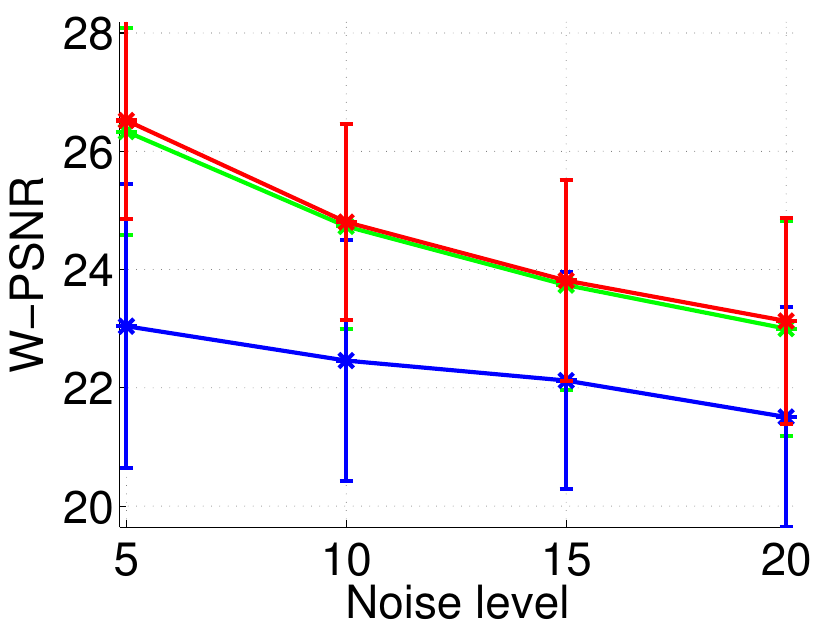} 
	\includegraphics[width=\figwidth\textwidth]{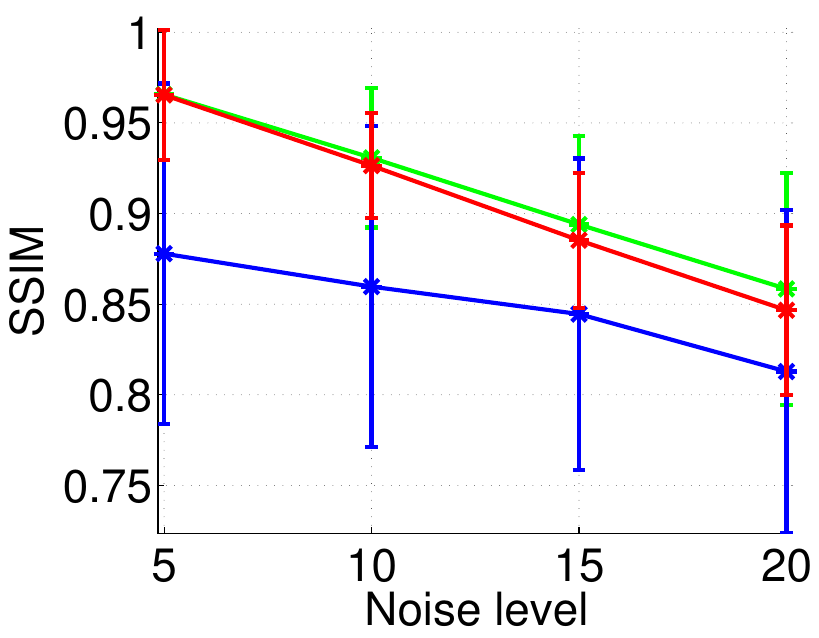}
	\includegraphics[width=\figwidth\textwidth]{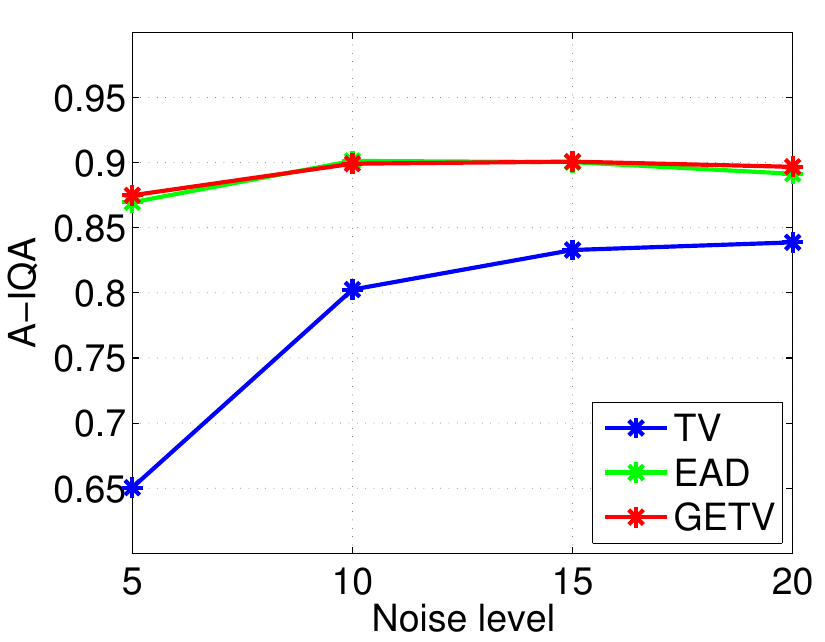} 
	\textbf{(b)} Mean and standard deviation of error measures for the Berkeley dataset.
    \caption{Error measures, a higher W-PSNR indicates better recovery of high-frequency regions. The visual appearance for selected images is shown in figure~\ref{fig:result_berkeley}. }
    \label{fig:error_measures}
\end{figure}

\renewcommand{\figtablewidth}{2.9cm}
\renewcommand{\figwidth}{0.24}
\begin{figure}[t]
    \centering
	\begin{tabular}{m{\figtablewidth} m{\figtablewidth} m{\figtablewidth} m{\figtablewidth}}
	 \footnotesize
		 Noisy & EAD & GETV & TV \\
		 \hline
	 \end{tabular} \\
	\includegraphics[width=\figwidth\textwidth]{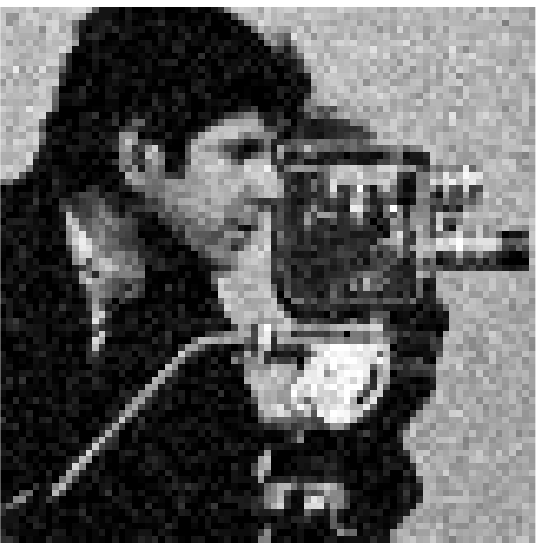}
	\includegraphics[width=\figwidth\textwidth]{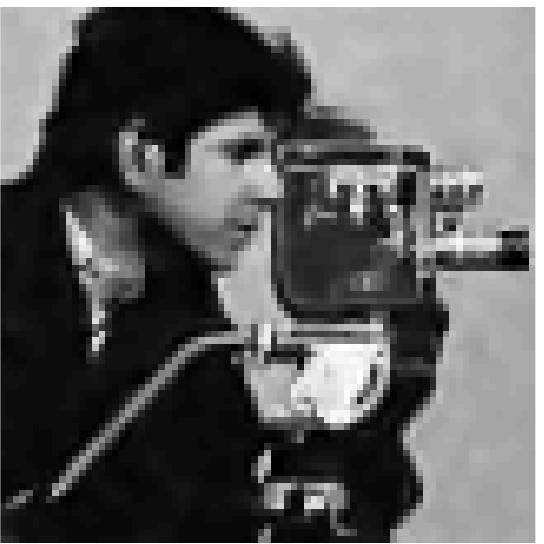}
	\includegraphics[width=\figwidth\textwidth]{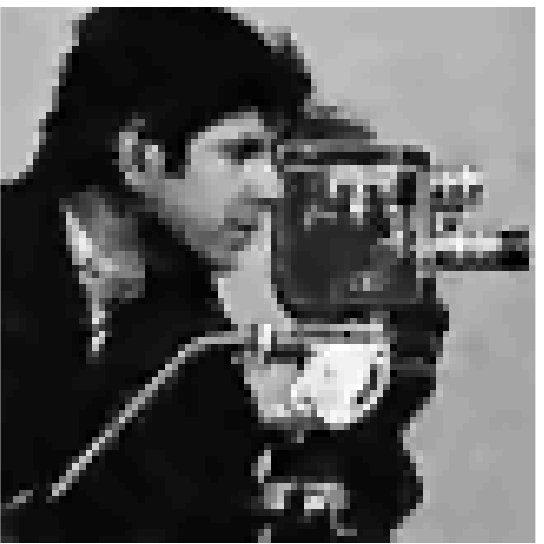} 
	\includegraphics[width=\figwidth\textwidth]{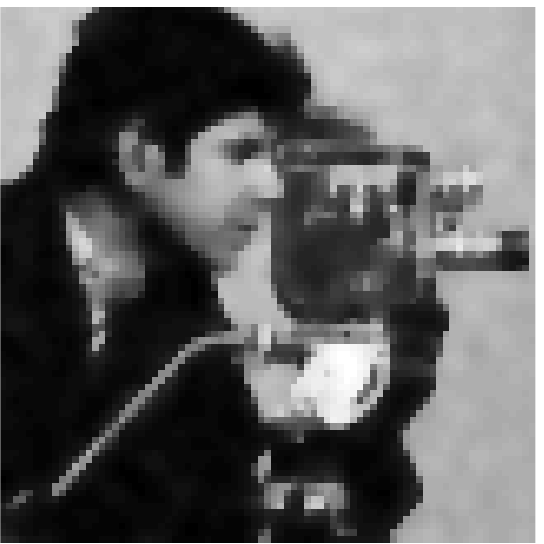} \\
	\begin{tabular}{m{\figtablewidth} m{\figtablewidth} m{\figtablewidth} m{\figtablewidth}}
		\small
		\input{cameraman_std15error_table.tex} \\
		\hline
	\end{tabular} \\
	\caption{Example from the grey-scale dataset for 15 standard deviation of noise. GETV shows an improvement over EAD and TV, both in PSNR and preservation of fine image structures such as the camera handle. Also, note that the images obtained with GETV appear less blurry than EAD and TV.}
    \label{fig:result_grayscale}
\end{figure}

\renewcommand{\figtablewidth}{2.9cm}
\renewcommand{\figwidth}{0.24}
\begin{figure}[t]
    \centering
	\begin{tabular}{m{\figtablewidth} m{\figtablewidth} m{\figtablewidth} m{\figtablewidth}}
		Noisy & EAD & GETV & TV \\
		\hline
	\end{tabular} \\
	\includegraphics[width=\figwidth\textwidth]{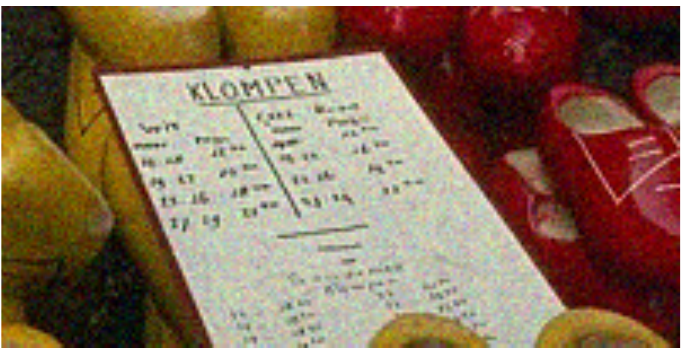}
	\includegraphics[width=\figwidth\textwidth]{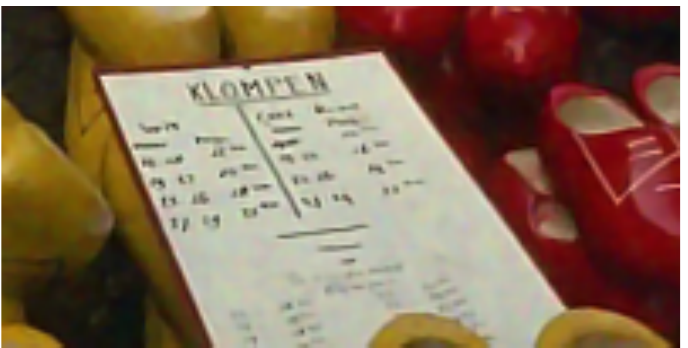}
	\includegraphics[width=\figwidth\textwidth]{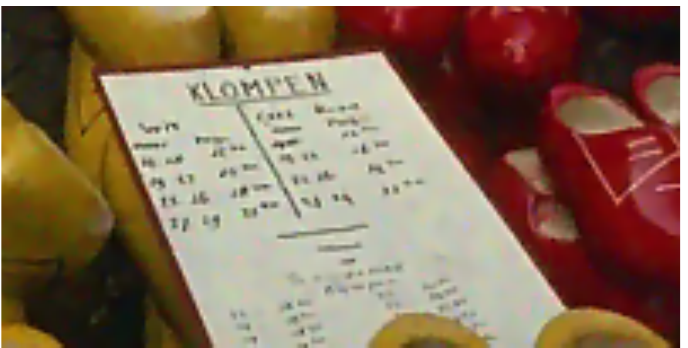} 
	\includegraphics[width=\figwidth\textwidth]{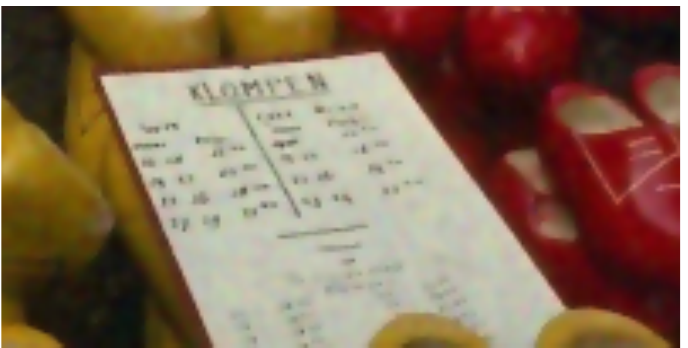} \\
	\begin{tabular}{m{\figtablewidth} m{\figtablewidth} m{\figtablewidth} m{\figtablewidth}}
		\input{140075_std15error_table.tex} \\
		\hline
	\end{tabular}\\
	\includegraphics[width=\figwidth\textwidth]{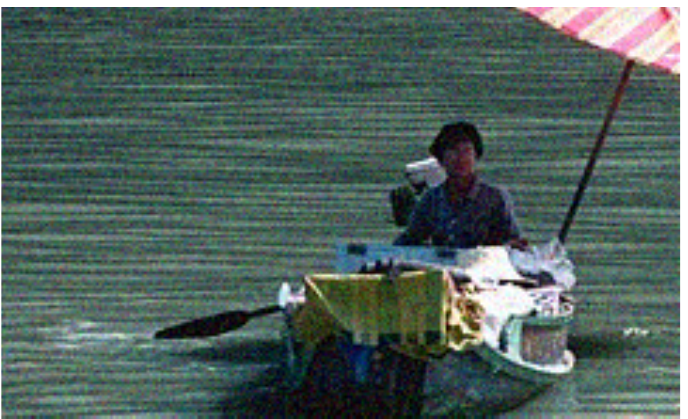}
	\includegraphics[width=\figwidth\textwidth]{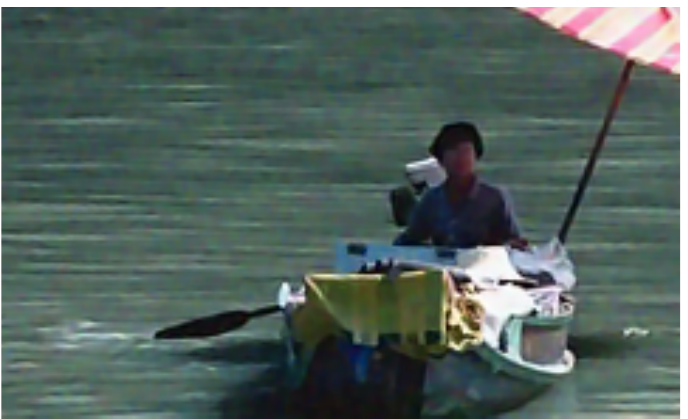}
	\includegraphics[width=\figwidth\textwidth]{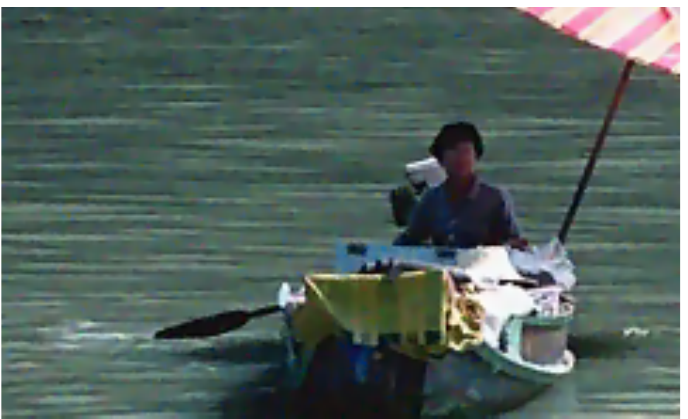} 
	\includegraphics[width=\figwidth\textwidth]{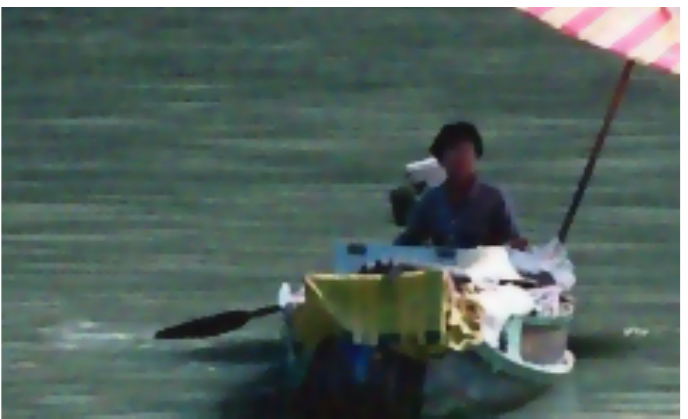} \\
	\begin{tabular}{m{\figtablewidth} m{\figtablewidth} m{\figtablewidth} m{\figtablewidth}}
		\input{147021_std15error_table.tex} \\
		\hline
	\end{tabular} \\
	\includegraphics[width=\figwidth\textwidth]{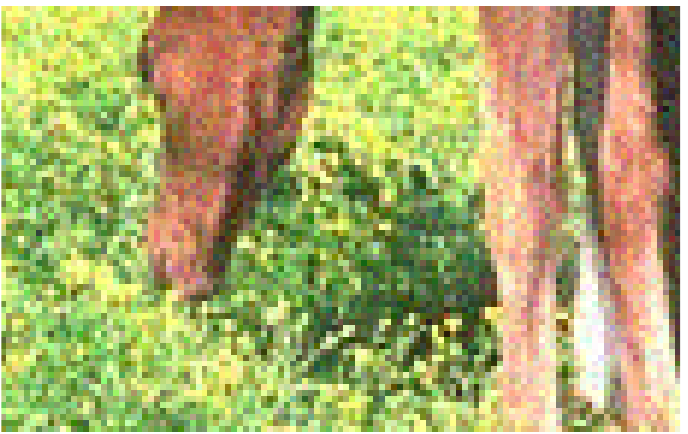}
	\includegraphics[width=\figwidth\textwidth]{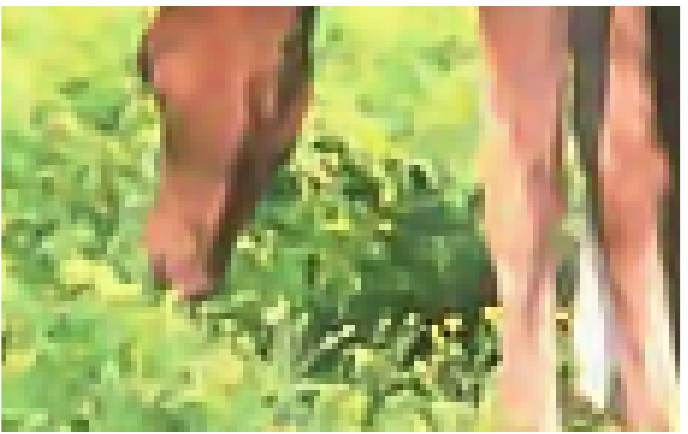}
	\includegraphics[width=\figwidth\textwidth]{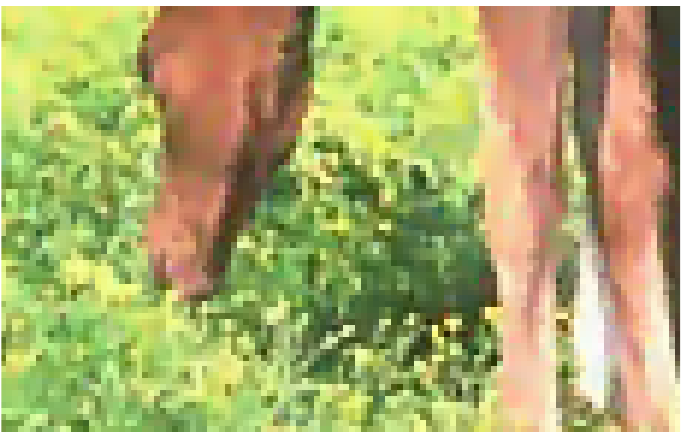} 
	\includegraphics[width=\figwidth\textwidth]{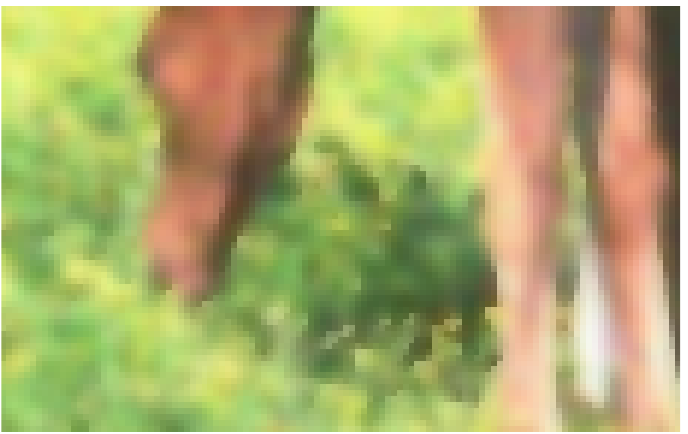} \\
	\begin{tabular}{m{3cm} m{3cm} m{3cm} m{3cm}}
		\input{113044_std15error_table.tex} \\
		\hline
	\end{tabular} \\
   \caption{Results from the Berkeley colour-image dataset with 15 standard deviation of noise where GETV excels. Consider particularly the text on the document and the grass behind the horse on the last row. Note that GETV in general preserves more fine details than EAD and TV, which both tends to oversmooth the images.}
    \label{fig:result_berkeley}
\end{figure}

\subsection{Experiments' setup}
\textbf{Datasets} that we consider are twofold. First we consider a number of standard grey-scale images \emph{barbara, boat, cameraman (cman), house} and \emph{lena}, each image is of size $256 \times 256$.  The other dataset is the Berkeley database~\cite{MartinFTM01}, where we randomly choose 50 colour images each of size $481 \times 321$. In the evaluation we corrupted each image with Gaussian additive noise of standard deviations 5, 10, 15 and 20. The images that we have used are listed in the supplementary material. In this work we use the decorrelation CIELAB transform, however other choices of colour spaces are possible as investigated in~\cite{diva2:543914}.

\textbf{Error measures} are in the image processing community recognised to not correlate with perceived image quality, therefore we investigate several error measures and consider the visual image quality. PSNR is widely used in the denoising literature so we report it, as well as, the structure similarity index (SSIM)~\cite{Wang2004} known to better reflect the true image quality. Also, since large homogeneous regions have more impact on the error measures than edges in the image do, we compute a weighted PSNR, W-PSNR, to assess preservation of edges in the images after filtering. The weight we use is given by the trace of the structure tensor and it is applied on the difference between the original and the enhanced image in the computation of the PSNR. Since the trace measures the magnitude of the gradient, the W-PSNR value correlates with a better preservation of edges than the PSNR measure does in relation to the noise-free image.

\textbf{Image auto-denoising} is used to optimize the selection of parameters in the different filtering schemes. It is a method which does not take the noise-free image into account when determining a quality measure. In this work we use the image auto-denoising metric proposed in~\cite{iccv-13-xiangfei-kong}, which we denote as A-IQA (auto-image quality assessment). The basic idea of A-IQA is that a high correlation score is obtained if the denoised image has smooth surfaces, but yet preserves boundaries. In the total variation scheme, we select the regularization parameter $\lambda$ from the values 6, 8, 10, 12 and 14 based on maximum A-IQA. The control parameter $k$ in the EAD scheme is computed according to $k = (\exp(1) - 1)/(\exp(1) - 2)\sigma^2$~\cite{Fberg2011} where $\sigma$ is the standard deviation of the added noise. The $k$ obtained for EAD is also used in the proposed GETV scheme but scaled with a factor $10^{-1}$. The stopping time for all methods was determined by the maximum A-IQA value.

\subsection{Result of image denoising}
In figure~\ref{fig:error_measures} (a) we show the PSNR values that we have obtained for each grey-scale image and noise level. We observe that the standard TV formulation does not perform well compared to EAD and GETV in these cases. In figure~\ref{fig:result_grayscale} we show close-ups of \emph{cameraman}. We note that in all cases the error measures are similar for the A-IQA values, however considering the visual quality it is obvious that more details are preserved in GETV, \ie the presence of sharp edges in the \emph{cameraman} image such as the handle of the camera. 

With respect to the colour images, figure~\ref{fig:result_berkeley} shows examples from the Berkeley dataset and the corresponding error measures are given in figure~\ref{fig:error_measures} (b). By comparing EAD and GETV for lower noise levels (5-15 standard deviations) we see that the difference in PSNR and SSIM is at best marginal. However, considering the variance, GETV is more robust than EAD. In figure~\ref{fig:result_berkeley} the visual differences can be seen for some selected images. Note that it is primarily in the high-frequency regions that GETV excels, consider \eg the clarity of the document, the visibility of waves and details in the grass in the horse image. For both gray-scale and colour images, EAD tends to oversmooth the images. Furthermore, it is obvious that the TV-method fails to handle these images when auto-tuning is used. By manually tweaking the regularization parameter of the methods we can improve the error measures for some images, however this approach is infeasible for a large amount of images. 

\section{Conclusion}
In this work we have presented a novel variational approach to tensor-based total variation. In particular, we have proposed a \emph{gradient energy total variation} functional which uses the gradient energy tensor. Our results suggest that the GETV formulation is suitable for images containing high-frequency information such as fine structures. Secondly, we showed by using the error measure A-IQA that the diffusion formulation performs well in denoising applications compared to EAD and TV.

\vspace{1mm}
\noindent {\bf Acknowledgement}. This research has received funding from the Swedish Foundation for Strategic Research through the grant VPS and from Swedish Research Council through grants for the projects energy models for computational cameras (EMC$^2$), Visualization-adaptive Iterative Denoising of Images (VIDI) and Extended Target Tracking (ETT), all within the Linnaeus environment CADICS and the excellence network ELLIIT.

\bibliographystyle{splncs03}
\bibliography{egbib}

\end{document}